%% file: main.tex
\documentclass[lettersize,journal]{IEEEtran}

\input{packages.tex}

\input{macros.tex}

\begin{document}

\title{Online, Reachability-Aware, Sampling-Based Motion Planning}
\author{Brendan Gould,~\IEEEmembership{Student Member,~IEEE}, Zhiyuan Zhang,~\IEEEmembership{Student Member,~IEEE}, Panagiotis Tsiotras,~\IEEEmembership{Fellow,~IEEE}, Samuel Coogan,~\IEEEmembership{Senior Member, IEEE}
}



\maketitle

\begin{abstract}
Sampling-Based Model-Predictive Control (MPC) algorithms are a flexible class of controllers used for navigation on a wide range of robotic systems. 
Historically, such approaches have lacked hard safety guarantees, a shortcoming which we remedy in this work by computing guaranteed reachable-set overapproximations \emph{online} with a fast, interval-based pipeline. 
We show that our method achieves similar performance to a state-of-the-art reachability-based planner without the need for the expensive pre-computation step, and can be scaled to systems that are infeasible using existing approaches. 
Finally, we demonstrate that our technique reduces safety violations by over 99\% in a racing simulation and successfully controls a model racecar on real hardware experiments without crashes. 
\end{abstract}


\section{Introduction}
\label{sec:intro}
\IEEEPARstart{M}{odern} robotic systems, such as autonomous vehicles, warehouse robots, and drones, are expected to demonstrate a significant degree of independence in navigating through their environment. 
For many of these systems, a simple navigation failure may result in unacceptable danger or even harm to humans and other robots. 
Designing controllers for such \emph{safety-critical} systems therefore requires additional care to ensure that catastrophic failures can be prevented, both in theory and in practice. 

One popular framework used to tackle the problem of navigation and planning is Model-Predictive Control (MPC)~\cite{schwenzerReviewModelPredictive2021}, which uses a model of the system dynamics to predict future states and select inputs that induce desirable trajectories.
This approach is notable for its flexibility; compared to classical approaches such as the linear-quadratic regulator, MPC places significantly less restrictive assumptions on the system dynamics and cost function. 
MPC algorithms have been applied to successfully control a wide variety of systems, including quadrotors~\cite{romeroModelPredictiveContouring2022}, legged robots~\cite{mastalliCrocoddylEfficientVersatile2020}, and aquatic vehicles~\cite{liGaussianProcessBased2024}.
One popular approach of MPC is the Model-Predictive Path Integral (MPPI) algorithm~\cite{williamsInformationTheoreticModelPredictive2018}, which controls a system by randomly sampling a set of control signals up to a finite horizon and assigning weights to each sample to minimize the Kullback-Leibler divergence between the sampled and optimal distributions. 

Standard MPPI does not consider safety constraints, making it unsuitable for use with safety-critical systems.
Prior work has proposed several methods to address this shortcoming. 
Some works incorporate safety into the cost function by penalizing trajectories that violate constraints or incur high levels of risk~\cite{yinRiskAwareModelPredictive2023,zhangRAPAPlannerRobustEfficient2025,mohamedEfficientMPPITrajectory2025}. 
Others attempt to modify the distribution of sampled controls to balance safety and performance~\cite{wangMppiDbasSafeTrajectory2025}. 
These techniques mitigate failure probabililty, but do not provide deterministic safety guarantees.
Alternatively, some approaches append a safety \emph{filter} (such as a control barrier function (CBF)~\cite{amesControlBarrierFunctions2019}) to convert a possibly unsafe control into a safe one after sampling and weighting~\cite{yinShieldModelPredictive2023}. 
This idea is theoretically appealing, but designing a CBF for practical systems is often non-obvious, and filter interventions can affect performance in counter-intuitive ways~\cite{mestresControlBarrierFunctionbased2026}

Reachable-set over-approximation provides rigorous safety guarantees by bounding the set of states the system can reach at a given time~\cite{harwoodEfficientPolyhedralEnclosures2016,ARCH15:Introduction_CORA_2015,bansalHamiltonJacobiReachabilityBrief2017}. 
Historically, such techniques have been too computationally expensive for online planning, and used instead for offline verification.
A recent toolbox, \texttt{immrax}~\cite{immrax}, was designed for fast and scalable reachability and addresses this issue with a combination of theoretical advancements~\cite{gouldAutomaticScalableSafety2025,harapanahalliParametricReachableSets2025} and a state-of-the-art compuational implementation using the high-performance computing library JAX~\cite{jax2018github}.
This, and other advances, have brought reachability-aware sampling planners within reach. 

Recent work has started to take advantage of this possibility, most notably~\cite{borquezDualGuardMPPISafe2025} which leverages Hamilton-Jacobi (HJ) reachability to guarantee safety under bounded disturbances. 
To do this, it pre-computes the HJ value function on a discretized grid over the state space, and interpolates between grid points at runtime to identify both the best safety-preserving control at any state and the last moment in time when such interventions can guarantee collision avoidance. 
The main weakness of such an approach is the need to solve a partial differential equation in advance, which becomes infeasible as system state dimension grows due to the curse of dimensionality~\cite{borquezDualGuardMPPISafe2025} and, in addition, assumes that the system and environment are fully known beforehand. 

In this work, we provide a similar guarantee to that of~\cite{borquezDualGuardMPPISafe2025} while removing the expensive pre-computation. 
Through conservative, but fast, interval-based reachability using \texttt{immrax}, our algorithm performs all reachability calculations at runtime. 
The contributions of this work are threefold: 
\begin{itemize}
    \item A novel reachable-set aware sampling-based controller and its accompanying safety guarantee (Section~\ref{subsec:interval_motion_planning}). 
    \item A thorough empirical validation of the proposed controller both in simulation (Section~\ref{subsec:simulation}) and on a \nicefrac{1}{28}th scale vehicle platform (Section~\ref{subsec:real_world}).
    \item A two-way comparison to the state-of-the-art sampling-based planner~\cite{borquezDualGuardMPPISafe2025} (Section~\ref{subsec:comparison}), where our algorithm demonstrates similar results without any pre-computation. 
\end{itemize}

\subsection{Notation}
\label{subsec:Notation}

For vectors $x, y \in \R^n$, we use the element-wise partial order $\le$. 
Whenever $x \le y$, let $\interval{x}{y} = \{ z \in \R^n \mid x \le z \le y \}$, and $\IR^{n}$ be the set of all such intervals.
The $\min$ and $\max$ operators on a finite collection of vectors are defined component-wise, and we abbreviate
\begin{equation}
    \label{eq:def_clip}
    \clip\left(x; \interval{x}\right) := \min\{\max\{ x, \lb{x} \}, \ub{x} \}.
\end{equation}
For any compact set $K \subseteq \mathbb{R}^n$, let $\dist(x, K)$ be the minimum distance between $x$ and any point in $K$, and let $\boundary K$ be the boundary of $K$. 
For a continuous time function $p : [a, b] \to \R^n$, we use $\trajC{p}{c}$ to denote its value at $c \in [a, b]$, and $\trajC{p}$ to denote the entire function. 
Similarly, we use $\trajD{p}{c}$ and $\trajD{p}$ for discrete time signals.
With a mild abuse of notation, we use the same symbol for both a discrete time function and the continuous time analogue formed by applying a zero-order hold between the function arguments.

\section{Problem Statement}
\label{sec:problem_statement}

Consider a continuous time, non-linear system with state $\state \in \R^\stateDim$, control input $\control \in \R^\controlDim$, and disturbance input $\err \in \R^\errDim$ subject to the dynamics
\begin{equation}
    \label{eq:general_dynamics}
    \dot{\state} = \dyn{\state}{\control}{\err}.
\end{equation}
The disturbance $\err$ is used to explicitly consider possible differences between an idealized, mathematical model and the real-world system.
For a given time horizon $\timeHorizon$, input \emph{signals} $\trajC{\control} \in \functionSet{\interval{0}{\timeHorizon}}{\R^\controlDim}$ and $\trajC{\err} \in \functionSet{\interval{0}{\timeHorizon}}{\R^\errDim}$ are Lebesgue measurable functions from any time instant before that horizon to a particular input value. 

\begin{assumption}
    \label{asm:traj_exist_unique}
    The vector field $\dyn$ satisfies the necessary regularity conditions such that the solutions to~\eqref{eq:general_dynamics} exist and are unique on the time interval $\interval{0}{\timeHorizon}$ for all initial conditions $\stateInit \in \R^\stateDim$ and input signals $\trajC{\control}$, $\trajC{\err}$.
    We denote this solution by $\traj{\timeInstant}{\stateInit}{\trajC{\control}}{\trajC{\err}}$.
\end{assumption}

Additionally, as we treat the disturbances as adversarial in nature, we require them to be bounded.
\begin{assumption}
    \label{asm:bounded_disturbance}
    There exists an interval $\interval{\err}$ such that $\trajC{\err}{\timeInstant} \in \interval{\err}$ for all $\timeInstant \in \interval{0}{\timeHorizon}$.
\end{assumption}

Our design goal is twofold and consists of both safety and performance components. 
To model safety, we use a function $\safetyFn : \R^\stateDim \to \R$ encoding a set of safe states 
\begin{equation}
    \label{eq:safeStates}
    \safeStates := \left\{ \state \mid \safetyFn(\state) \ge 0 \right\}.
\end{equation}
This framework is particularly suited for e.g. collision avoidance, where states with $\safetyFn{\state} \ge 0$ are not in collision and states with $\safetyFn{\state} < 0$ are.
At any $\stateInit$, safe control signals are those that produce state trajectories contained in $\safeStates$ for the next $\timeHorizon$ time units under any realization of $\err$: 
\begin{align}
    \label{eq:safeControls}
    \safeControls{\stateInit} & \defeq \bigg\{ \trajC{\control} \mid \nonumber \\
    & \quad \inf_{\trajC{\err}, \timeInstant \in \interval{0}{\timeHorizon}} \safetyFn{\traj{\timeInstant}{\stateInit}{\trajC{\control}}{\trajC{\err}}} \ge 0 \bigg\}.
\end{align}
Our primary design goal is to construct a safe control. 
\begin{problem}
    \label{prob:safe_control}
    Given the current system state $\stateInit$, identify a safe control signal $\trajC{\control} \in \safeControls{\stateInit}$, or verify that $\safeControls{\stateInit}$ is empty.
\end{problem}

In addition to satisfying safety constraints, we often wish to design controls that achieve some objective. 
We model such objectives with a cost functional $\costFunctional : \functionSet{\interval{0}{\timeHorizon}}{\R^\stateDim} \times \functionSet{\interval{0}{\timeHorizon}}{\R^\controlDim} \to \R$ that indicates which states and control actions are desirable. 
\begin{problem}
    \label{prob:optimal_control}
    Given the current system state $\stateInit$, select an optimal control from the set of safe controls. 
    That is, find
    \begin{equation}
        \label{eq:optimal_control}
        \trajC{\controlRef} \in \argmin_{\trajC{\control} \in \safeControls{\stateInit}} \costFunctional{\traj{\cdot}{\stateInit}{\trajC{\control}}{\trajC{0}}}{\trajC{\control}}.
    \end{equation}
\end{problem}

In the next section, we present our solution to both of these problems.


\section{Methodology}
\label{sec:methodology}

This section presents our theoretical contributions, which are organized as follows. 
Section~\ref{subsec:reachability} summarizes the established reachability technique used in our work, and Section~\ref{subsec:motion_planning} introduces a sampling-based motion planner. 
These ideas are combined in Section~\ref{subsec:interval_motion_planning}, which contains both our proposed controller (Algorithm~\ref{alg:reachset_planner}) and its corresponding safety guarantee (Theorem~\ref{thm:robustly_safe_control}). 
Finally, Section~\ref{subsec:limitations} describes remaining opportunities for future work to better address Problems~\ref{prob:safe_control}--\ref{prob:optimal_control}.

\subsection{Interval Reachability}
\label{subsec:reachability}

We leverage an interval-based reachability method developed in prior work~\cite{shenRapidAccurateReachability2017}; it is summarized here for convenience. 
This method uses an \emph{embedding system} whose trajectories define a time-varying interval set that is guaranteed to contain the trajectories of the original system, subject to bounded disturbances. 
Because of the simplicity of intervals, these bounds can be computed efficiently, at approximately 2\texttimes{} the cost of a trajectory of the original system. 

The primary technique used to over-approximate the effect of disturbances on system trajectories is that of an inclusion function from interval arithmetic~\cite{jaulinAppliedIntervalAnalysis2001}. 
\begin{definition}[Inclusion Function]
    \label{def:inclusion_fn}
    Let $f : \R^m \to \R^n$. 
    We say $\interval{f} : \IR^m \to \IR^n$ is an \emph{inclusion function} for $f$ if for any $x \in \interval{x}$, $f(x) \in \interval{f(\interval{x})}$.
    This notation has three components; a lower bound $\lb{f} : \IR^m \to \R^n$, an upper bound $\ub{f} : \IR^m \to \R^n$, and a combined $\interval{f} : \IR^m \to \IR^n$ which outputs the interval between the lower and upper bounds. 
\end{definition}
Definition~\ref{def:inclusion_fn} generalizes to functions with multiple inputs naturally, requiring containment of the output for any combination of inputs within the specified intervals. 
Importantly, inclusion functions are not required to tightly contain their outputs (they may be conservative), nor unique. 

As intuition may suggest, it is possible to over-approximate the trajectories of~\eqref{eq:general_dynamics} by simply replacing $\dyn$ with one of its inclusion functions. 
However, a low-cost additional step can be used to greatly reduce the conservatism of such approximations by leveraging the structure of the continuous time dynamics. 
Geometrically, if any trajectory were to exit an interval $\interval{\state}$, it must first pass through one of the faces of $\interval{\state}$. 
We introduce the following notation to describe these faces:
\begin{align}
    \lb{\flatten{\componentIndex}{\interval{\state}}} & := \{\state \mid \state \in \interval{\state} \land \component{\state} = \lb{\state}_k \}, \\
    \ub{\flatten{\componentIndex}{\interval{\state}}} & := \{\state \mid \state \in \interval{\state} \land \component{\state} = \ub{\state}_k \}.
\end{align}

Then, instead of over-approximating the expansion due to $\dyn$ over the entire state interval, we need only consider how quickly states that lie on a face of $\interval{\state}$ are expanding. 
\begin{definition}[Embedding System]
    \label{def:embedding_sys}
    Consider the dynamical system~\eqref{eq:general_dynamics}. 
    Let $\interval{\dyn}$ be an inclusion function for $\dyn$. 
    The corresponding \emph{embedding system} has state $\interval{\state}$ with dynamics given by 
    \begin{equation}
    \label{eq:dyn_embedding}
        \begin{bmatrix}
            \dot{\lb{\state}}_k \\ 
            \dot{\ub{\state}}_k
        \end{bmatrix} =
        \begin{bmatrix}
            \lb{\component{\dyn{\lb{\flatten{\componentIndex}{\interval{\state}}}}{\interval{\control}{\control}}{\interval{\err}}}} \\
            \ub{\component{\dyn{\ub{\flatten{\componentIndex}{\interval{\state}}}}{\interval{\control}{\control}}{\interval{\err}}}}
        \end{bmatrix},
    \end{equation}
    for all $\componentIndex \in \{1, \ldots, \stateDim\}$.
\end{definition}
Note that each face of an interval is itself an interval set, so~\eqref{eq:dyn_embedding} is well-defined.
Let $\embTraj{\timeInstant}{\stateInit}{\trajC{\control}}$ denote the solution of~\eqref{eq:dyn_embedding} from the initial condition $\interval{\stateInit}{\stateInit}$, where $\lb{\LBname{\Phi}}$ and $\ub{\LBname{\Phi}}$ individually give the lower and upper bounds of the interval trajectory.
We have suppressed the dependence on error bounds $\interval{\err}$ in the notation to emphasize the fact that such trajectories are independent of any \emph{realization} of error signal within those bounds. 
As desired, the time varying interval $\interval{\LBname{\Phi}}$ is guaranteed to contain trajectories of the original system under bounded disturbances.

\begin{proposition}[{\cite[Theorem 1]{shenRapidAccurateReachability2017}}]
    \label{prop:embedding_containment}
    For any initial state $\stateInit$, control signal $\trajC{\control} \in \functionSet{\interval{0}{\timeHorizon}}{\R^\controlDim}$, and disturbance signal $\trajC{\err} \in \functionSet{\interval{0}{\timeHorizon}}{\R^\errDim}$ satisfying Assumption~\ref{asm:bounded_disturbance}, we have, for all $\timeInstant \in \interval{0}{\timeHorizon}$, 
    \begin{equation}
        \label{eq:embedding_containment}
        \traj{\timeInstant}{\stateInit}{\trajC{\control}}{\trajC{\err}} \in \embTraj{\timeInstant}{\stateInit}{\trajC{\control}}.
    \end{equation}
\end{proposition}

\begin{remark}
    \label{rmk:immrax_auto_if}
    Application of Proposition~\ref{prop:embedding_containment} requires the computation of an inclusion function $\interval{\dyn}$ for the system dynamics. 
    The software toolbox \texttt{immrax}~\cite{immrax} can automatically compute such a $\interval{\dyn}$ for any dynamics that are \emph{factorable} into compositions of operations in a library of primitive mathematical functions. 
    In Section~\ref{sec:experiments}, we use \texttt{immrax} to construct $\interval{\dyn}$ for a 6D dynamic bicycle vehicle model with integrated engine and tire slip models. 
    We also demonstrate heuristics that can be used as computationally speedy alternatives for operations that are expensive to extend to intervals. 
\end{remark}

\subsection{Trajectory-Based Motion Planning}
\label{subsec:motion_planning}

Our algorithm follows a receding-horizon, sampling-based paradigm inspired by those used in prior work~\cite{williamsInformationTheoreticModelPredictive2018,yinShieldModelPredictive2023,borquezDualGuardMPPISafe2025}. 
As a prelude to our proposed, reachability-aware planner, we now present a simplified approach. 
This section serves to establish notation and provide a benchmark comparison (see Section~\ref{subsec:simulation}) that isolates the benefits of reachability analysis. 

First, $\numControlSamples$ candidate control signals are sampled from a Gaussian distribution centered around a reference $\trajD{\controlRef}$.
We consider piecewise-constant policies with uniform time discretization $\controlDt$ and $\predictionHorizon$ segments, for a prediction horizon of $\timeHorizon = \predictionHorizon \controlDt$.
After sampling, the controls are smoothed and clipped to enable low-level tracking.
For precise implementation details, please see the project \href{https://github.com/3cM2VCvv/2026-RAL-sampling}{github repository}. 

Once a batch of controls is sampled, the system dynamics~\eqref{eq:general_dynamics} are integrated for each sample (assuming zero disturbance), yielding a set of nominal state trajectories. 
Each nominal trajectory is checked for predicted safety violations and the remaining trajectories are evaluated using the cost functional:
\begin{align}
    \maybeSafeSamples{\stateInit}  \defeq &\bigg\{ \sampleIndex \in \{1, \ldots, \numControlSamples\} \mid \nonumber \\
    & \hspace*{-5mm} \inf_{\timeInstant \in \interval{0}{\timeHorizon}} \safetyFn{\traj{\timeInstant}{\stateInit}{\sample{\trajD{\control}}}{\trajC{0}}} \ge 0 \bigg\} \label{eq:maybe_safe_samples} \\
  \sampleIndexReference \defeq \argmin_{\sampleIndex \in \maybeSafeSamples{\stateInit}} &\costFunctional{\traj{\timeInstant}{\stateInit}{\sample{\trajD{\control}}}{\trajC{0}}}{\sample{\trajD{\control}}} \label{eq:min_cost_selection}
\end{align}
Finally, the control differences corresponding to the optimal sampled control $\sample{\trajD{\controlDiff}}{\sampleIndexReference}$ 
are saved and used as the reference to ``warm-start'' the next control loop. 
The full strategy is given in Algorithm~\ref{alg:sample_planner}.


\begin{algorithm}[H]
\caption{Sampling-Based Planner (Trajectories)}
\label{alg:sample_planner}
\begin{algorithmic}[1]
\Require Reference control sequence $\trajD{\controlRef}$
\Function{Sample}{$\trajD{\controlRef}$}
    \State $\trajD{\controlRef} \gets \begin{bmatrix} \trajD{\controlRef}{2:\predictionHorizon} & 0 \end{bmatrix}$ \Comment{progress time}
    \State $\sample{\trajD{\control}}{1:\numControlSamples} \gets$ sample controls around $\trajD{\controlRef}$
    \State \Return $\sample{\trajD{\control}}{1:\numControlSamples}$
\EndFunction
\Function{Rollout}{$\trajD{\control}$, $\stateInit$}
    \State \Return $\traj{\cdot}{\stateInit}{\trajD{\control}}{\trajC{0}}$ by Euler integration
\EndFunction
\Function{SelectBest}{$\sample{\trajC{\state}}{1:\numControlSamples}, \sample{\trajD{\control}}{1:\numControlSamples}$}
    \State $\maybeSafeSamples{\stateInit} \gets$ predicted safe samples from~\eqref{eq:maybe_safe_samples}
    \State \Return Optimal sample index as in~\eqref{eq:min_cost_selection}
\EndFunction
\State
\State $\sample{\trajD{\control}}{1:\numControlSamples} \gets$ \textproc{Sample}$(\trajD{\controlRef})$
\State $\sample{\trajC{\state}}{1:\numControlSamples} \gets$ \textproc{Rollout}$(\sample{\trajD{\control}}, \stateInit)$ for $\sample{\trajD{\control}} \in \sample{\trajD{\control}}{1:\numControlSamples}$
\State $\sampleIndexReference \gets$ \textproc{SelectBest}$\left(\sample{\trajC{\state}}{1:\numControlSamples}, \sample{\trajD{\control}}{1:\numControlSamples}\right)$
\State Apply $\sample{\trajD{\control}{1}}{\sampleIndexReference}$, save $\sample{\trajD{\control}}{\sampleIndexReference}$
\end{algorithmic}
\end{algorithm}

\subsection{Reachability-Aware Motion Planning}
\label{subsec:interval_motion_planning}

We now present our main theoretical contribution: a generalization of Algorithm~\ref{alg:sample_planner} that leverages the reachable sets produced by Proposition~\ref{prop:embedding_containment} to robustly guarantee safety in the presence of disturbances. 

Instead of integrating the dynamics directly and ignoring the possibility of disturbances, we compute trajectories of the \emph{embedding} system corresponding to each sampled control to over-approximate the reachable sets for \emph{any} realizable disturbance. 
This allows us to robustly check the safety of the sampled control, as opposed to~\eqref{eq:maybe_safe_samples}, which only considered the nominal trajectory.
\begin{theorem}
    \label{thm:robustly_safe_control}
    Let $\interval{\safetyFn}$ be any inclusion function for $\safetyFn$, and consider the set
    \begin{align}
        \safeSamples{\stateInit} & \defeq \bigg\{ \sampleIndex \in \{1, \ldots, \numControlSamples \} \mid \nonumber \\ 
        & \quad \inf_{\timeInstant \in \interval{0}{\timeHorizon}} \lb{\safetyFn{\embTraj{\timeInstant}{\stateInit}{\sample{\trajD{\control}}}}} \ge 0 \bigg\}. \label{eq:guaranteed_safe_samples}
    \end{align}
    All such sampled control sequences are safe controls: 
    \begin{equation}
        \label{eq:guaranteed_safety}
        \left\{ \sample{\trajD{\control}} \mid \sampleIndex \in \safeSamples{\stateInit} \right\} \subseteq \safeControls{\stateInit}.
    \end{equation}
\end{theorem}
\begin{proof}
    Consider any $\sample{\trajD{\control}} \in \safeSamples{\stateInit}$.
    Proposition~\ref{prop:embedding_containment} guarantees that its embedding system trajectory contains any solution of~\eqref{eq:general_dynamics} that could be realized by some disturbance signal. 
    Therefore, by Definition~\ref{def:inclusion_fn} we have, for all $\timeInstant \in \interval{0}{\timeHorizon}$,
    \begin{align*}
        0 & \le \lb{\safetyFn{\embTraj{\timeInstant}{\stateInit}{\sample{\trajD{\control}}}}} \\
          & \le \safetyFn{\traj{\timeInstant}{\stateInit}{\sample{\trajD{\control}}}{\trajC{\err}}},
    \end{align*}
    implying $\sample{\trajD{\control}} \in \safeControls{\stateInit}$. 
\end{proof}

Once the safety of a control signal has been verified by Theorem~\ref{thm:robustly_safe_control}, it can be applied directly with the knowledge that no disturbance will ever cause a safety violation. 
This bypasses the need to account for disturbances with an ``averaging step'' (as in MPPI~\cite{williamsInformationTheoreticModelPredictive2018}) and motivates the use of the best sampled signal that is safe. 
Furthermore, the predicted interval trajectories give a much richer information structure to quantify which controls are best through a cost functional $\costFunctionalInt : \functionSet{\interval{0}{\timeHorizon}}{\IR^\stateDim} \times \functionSet{\interval{0}{\timeHorizon}}{\R^\controlDim} \to \R$.
For example, given any inclusion function $\interval{\costFunctional}$, the choice
\begin{align}
    \costFunctionalInt{\trajC{\interval{\state}}}{\trajC{\control}} & \defeq \ub{\costFunctional{\trajC{\interval{\state}}}{\trajC{\interval{\control}{\control}}}}, \label{eq:interval_cost_functional} \\ 
    \sampleIndexReference = \argmin_{\sampleIndex \in \safeSamples{\stateInit}} & \costFunctionalInt{\embTraj{\cdot}{\stateInit}{\trajD{\control}}}{\trajD{\control}}. \label{eq:robust_min_cost_selection}
\end{align}
would optimize a guaranteed (possibly conservative) upper bound for the cost incurred by any disturbance realization. 

A reachable-set aware planner can then be easily constructed by replacing~\eqref{eq:maybe_safe_samples} with~\eqref{eq:guaranteed_safe_samples} and~\eqref{eq:min_cost_selection} with~\eqref{eq:robust_min_cost_selection} in Algorithm~\ref{alg:sample_planner}.
The resulting logic is outlined in Algorithm~\ref{alg:reachset_planner}.

\begin{algorithm}[H]
\caption{Sampling-Based Planner (Reachable Sets)}
\label{alg:reachset_planner}
\begin{algorithmic}[1]
\Require Reference control sequence $\trajD{\controlRef}$
\Function{RolloutReachset}{$\trajD{\control}$, $\stateInit$}
    \State \Return $\embTraj{\cdot}{\stateInit}{\trajD{\control}}$ by Euler integration
\EndFunction
\Function{SelectBestReachset}{$\sample{\trajC{\interval{\state}}}{1:\numControlSamples}$, $\sample{\trajD{\control}}{1:\numControlSamples}$}
    \State $\safeSamples{\stateInit} \gets$ guaranteed safe samples from~\eqref{eq:guaranteed_safe_samples}
    \State \Return Optimal sample index as in~\eqref{eq:robust_min_cost_selection}
\EndFunction
\State
\State $\sample{\trajD{\control}}{1:\numControlSamples} \gets$ \textproc{Sample}$(\trajD{\controlRef})$
\State $\sample{\trajC{\interval{\state}}}{1:\numControlSamples} \gets$ \textproc{RolloutReachset}$(\sample{\trajD{\control}}$, $\stateInit)$ for $\sample{\trajD{\control}} \in \sample{\trajD{\control}}{1:\numControlSamples}$
\State $\sampleIndexReference \gets $\textproc{SelectBestReachset}$\big(\sample{\trajC{\interval{\state}}}{1:\numControlSamples}$, $ \sample{\trajD{\control}}{1:\numControlSamples}\big)$
\State Apply $\sample{\trajD{\control}{1}}{\sampleIndexReference}$, save $\sample{\trajD{\control}}{\sampleIndexReference}$
\end{algorithmic}
\end{algorithm}

\begin{remark}
    Though the theoretical modifications introduced by Algorithm~\ref{alg:reachset_planner} are simple, their practical implementation was only recently made possible by modern technological advances.
    The software toolbox \texttt{immrax}~\cite{immrax} leverages both interval reachability techniques designed for speed and scalability and the computational features of the JAX platform~\cite{jax2018github}.
    In particular, automatic parallelization over a modern GPU is critical for the real-time performance of Algorithm~\ref{alg:reachset_planner}, allowing all sampled control sequences to be processed simultaneously. 
\end{remark}

\begin{remark}
    \label{rmk:observability}
    The subroutine $\textproc{RolloutReachset}$ in Algorithm~\ref{alg:reachset_planner} requires observing the system's current state, which may not always be feasible. 
    In this work, all hardware experiments were conducted in an environment with optitrack cameras that reported the state very precisely. 
    However, it would also be easy to generalize the reachability-aware Algorithm~\ref{alg:reachset_planner} to account for observation error by e.g. initializing the embedding system~\eqref{eq:dyn_embedding} with an interval that contains the true state with high confidence, rather than a single point. 
\end{remark}

\subsection{Limitations}
\label{subsec:limitations}

The safety guarantee provided by Theorem~\ref{thm:robustly_safe_control} rules out many forms of failure caused by disturbances. 
However, practical limitations remain in applying this result. 
In particular, it is possible that the sampling process in Algorithm~\ref{alg:reachset_planner} does not find any safe control sequence, even if one exists.
As only finitely many sequences are sampled, and since each sequence conservatively over-approximates the effect disturbances could have on the system state, there is no guarantee that a safe control is sampled and then proven safe.
Furthermore, as a receding horizon approach, control sequences in $\safeControls{\stateInit}$ may induce system states from which safety violations cannot be prevented, as long as such violations would occur after the prediction horizon $\timeHorizon$.
That is, Algorithm~\ref{alg:reachset_planner} does not have any recursive feasibility guarantee. 
These factors explain all safety violations observed during our experiments in Section~\ref{subsec:simulation}, and are discussed further there.  

Another potential failure mode is related to the method used to compute predicted trajectories, usually using numerical integration. 
If the integrator performs poorly and significantly diverges from the true trajectory, controls that generate unsafe trajectories may fail to be detected. 
Furthermore, evaluating the infimum in~\eqref{eq:maybe_safe_samples} and~\eqref{eq:guaranteed_safe_samples} may be difficult in practice, and can be heuristically substituted with checking every discretized timestep in the numerically integrated trajectory. 
If the discretizations are too large, trajectories that are unsafe for short periods of time might be falsely considered safe. 

Finally, computing trajectories of the embedding system~\eqref{eq:dyn_embedding} requires formalizing both a bound on the possible disturbance magnitude and the precise mathematical mechanism by which they affect system dynamics. 
If the true disturbances are larger than expected, or they differ from their modeled influence on system state, the safe control set $\safeControls{\stateInit}$ no longer accurately represents which control sequences are safe to apply. 
Therefore, one must carefully choose a reasonable disturbance model before applying Algorithm~\ref{alg:reachset_planner}.

\section{Numerical Simulations and Experiments}
\label{sec:experiments}

This section evaluates Algorithm \ref{alg:reachset_planner} in simulation and on real-world hardware. We first define the vehicle dynamics and planning objectives for our test system (Section \ref{subsec:dynamics}). Next, we benchmark the algorithm's performance and robustness in simulation, including a comparison against a state-of-the-art reachability-aware planner \cite{borquezDualGuardMPPISafe2025} (Sections \ref{subsec:simulation} and  \ref{subsec:comparison}). Finally, we demonstrate the controller's real-world viability on a physical 1/28th scale racecar (Section \ref{subsec:real_world}).


\subsection{Vehicle Dynamics and Objective}
\label{subsec:dynamics}
We evaluate our controller on a racecar with mass $\carMass$ and length $\carLenFront$ ($\carLenRear$) from the center of mass to the front (rear) axle.
We use a dynamic bicycle model \cite{rajamaniVehicleDynamicsControl2012} with states $\posX$, $\posY$, $\heading$, $\velLong$, $\velLat$, and $\velHeading$.
These represent the 2D position, heading, longitudinal and lateral velocities, and yaw rate, respectively.
The control inputs are steering $\inputSteer$ and acceleration (throttle and braking) $\inputThrottle$. 
Tire slip forces are modeled using the Pacejka Magic Formula \cite{pacejkaMAGICFORMULATYRE1992}:
\begin{align}
    \tireSlipCurve{s} & := \tireSlipPeak \sin(\tireSlipShape + \arctan(\tireSlipStiffness s)), \label{eq:tire_slip_model} \\ 
    \carSlipFront & := \carLenRear \frac{\carMass \gravity}{\carLenFront + \carLenRear} \tireSlipCurve{-\arctan\left(\frac{\velHeading \carLenFront + \velLat}{\velLong}\right) + \inputSteer}, \label{eq:car_slip_front} \\
    \carSlipRear & := \carLenFront \frac{\carMass \gravity}{\carLenFront + \carLenRear} 1.15 \tireSlipCurve{\arctan\left(\frac{\velHeading \carLenRear - \velLat}{\velLong}\right)}. \label{eq:car_slip_rear} 
\end{align}
Input acceleration acts longitudinally on the vehicle according to a motor model:
\begin{equation}
    \label{eq:motor_model}
    \motorModel{\velLong, \inputThrottle} := \motorAccWeight \left(\inputThrottle - \frac{\velLong}{\motorVelEfficiency} - \motorEquilibrium\right).
\end{equation}

The complete system dynamics are therefore given by:
\begin{equation}
    \label{eq:dyn_dynamic}
    \begin{bmatrix}
        \dot{\posX} \\ 
        \dot{\posY} \\
        \dot{\heading} \\
        \dot{\velLong} \\
        \dot{\velLat} \\
        \dot{\velHeading} 
    \end{bmatrix} = 
    \begin{bmatrix}
        \velLong \cos(\heading) - (\velLat + \distVelLat) \sin(\heading) \\
        \velLong \sin(\heading) + (\velLat + \distVelLat) \cos(\heading) \\
        \velHeading + \distHeading \\
        \motorModel{\velLong, \inputThrottle} \\ 
        \carMass^{-1} (\carSlipFront + \carSlipRear - \carMass \velLong \velHeading) \\ 
        \carRotationalInertia^{-1} (\carSlipFront \carLenFront - \carSlipRear \carLenRear)
    \end{bmatrix}.
\end{equation}
Here, $\distVelLat \in \interval{\distVelLat}$ and $\distHeading \in \interval{\distHeading}$ are bounded, adversarial disturbances. These terms account for the inherent difficulty of accurately estimating tire slip and lateral velocity, a process that typically requires successive numerical differentiation of position data and is highly susceptible to sensor noise \cite{liao2019adaptive, ungoren2004lateral}. Algorithm \ref{alg:reachset_planner} robustly selects control actions that remain safe for any realization of $\distVelLat$ and $\distHeading$.

For small $\velLong$, the calculation of the slip forces $\carSlipFront$ and $\carSlipRear$ in~\eqref{eq:car_slip_front}--\eqref{eq:car_slip_rear} becomes singular. 
To ensure numerical stability, we switch to a simplified, kinematic bicycle model when longitudinal speed passes below a threshold. 
Given the \emph{slip-slide} angle
\begin{equation}
    \label{eq:slip_slide}
    \beta := \arctan\left(\frac{\carLenRear}{\carLenFront + \carLenRear} \tan(\inputSteer)\right),
\end{equation}
the kinematic model predicts dynamics
\begin{equation}
    \label{eq:dyn_kinematic}
    \begin{bmatrix}
        \dot{\posX} \\ 
        \dot{\posY} \\
        \dot{\heading} \\
        \dot{\velLong} \\
        \dot{\velLat} \\
        \dot{\velHeading} 
    \end{bmatrix} = 
    \begin{bmatrix}
        \velLong \cos(\heading + \beta) \\
        \velLong \sin(\heading + \beta) + \distVelLat \\
        \frac{1}{\carLenFront + \carLenRear} \velLong \cos(\beta) \tan(\inputSteer) + \distHeading \\
        \motorModel{\velLong, \inputThrottle} \\ 
        -\velLat \\ 
        - (\velHeading - \dot{\heading})
    \end{bmatrix}.
\end{equation}
Here, $\distVelLat$ and $\distHeading$ are used similarly as in~\eqref{eq:dyn_dynamic} to account for modeling errors. 


\begin{remark}
\label{rmk:interval_model_switching}
There are several reasonable methods to extend the model switching motivated by the singularity in~\eqref{eq:car_slip_front}--\eqref{eq:car_slip_rear} to the interval-valued embedding system. 
Whenever the interval state does \emph{not} overlap the $\velLong$ threshold, the selection between~\eqref{eq:dyn_dynamic} and~\eqref{eq:dyn_kinematic} is unambiguous. 
If there is an overlap, one option is to evaluate embedding systems for both dynamics separately and use the smallest interval containing the result of both. 
In practice, we have noticed that this is needlessly conservative, and simply using the largest value of $\velLong$ in the input interval is sufficient. 
\end{remark}

Despite using Cartesian coordinates for reachability computations, the safety and performance criteria are still most easily represented in the Frenet reference frame. 
By converting to coordinates progress $\progress$, lateral error $\errLat$, heading error $\errHeading$, velocity components $\velLong$ and $\velLat$, and angular velocity error $\errVelHeading$ (all relative to a reference raceline), we can naturally specify notions such as ``avoid collisions'' and ``drive quickly''. 

\begin{assumption}[Track Observability]
    \label{asm:track_geometry}
    We assume that during each control computation, the planner has access to the approaching track geometry up to time horizon $\timeHorizon$.
    Specifically, this information is represented as a reference raceline $\trajC{\progressRef}$, the lateral distance to the left and right track boundaries $\interval{\errLat}$, and desired longitudinal velocity $\velLongRef$. 
\end{assumption}
Assumption~\ref{asm:track_geometry} is satisfied by the standard outputs of many common online observability and localization algorithms. 
In this work, and for simplicity of presentation, we use fixed environments where such information is readily available.
Once the track geometry is known, specifying safe (non-colliding) states is simple in the Frenet frame: 
\begin{equation}
    \label{eq:frenet_no_collision}
    \safetyFn{\state} \defeq -\operatorname{dist}\left(\errLat, \interval{\trajD{\errLat}{\progress}} \right).
\end{equation}

Similarly, we design a cost function to incentivize staying near the raceline and progressing quickly along it. 
We have four running cost terms: 
\begin{align}
    \costCentering{\trajC{\state}{\timeInstant}} &= \lvert \trajC{\errLat}{\timeInstant} \rvert, \\
    \costBoundary{\trajC{\state}{\timeInstant}} &= \max \bigg\{0, \arctan\big(-100 \times \nonumber \\ 
    ( \operatorname{dist}&(\errLat, \boundary \interval{\trajD{\errLat}{\trajC{\progress}{\timeInstant}}}) + 0.05 )\big) + \frac{\pi}{2} \bigg\}, \\
    \costVelRef{\trajC{\state}{\timeInstant}} &= \left(\trajC{\velLong}{\timeInstant} - \trajD{\velLongRef}{\trajC{\progress}{\timeInstant}}\right)^{2}, \\
    \costVelNeg{\trajC{\state}{\timeInstant}} &= - \min \{ \trajC{\velLong}{\timeInstant}, 0 \},
\end{align}
and a single terminal cost $\costProgress{\trajC{\state}} = \trajC{\progress}{0} - \trajC{\progress}{\timeHorizon}$.
These terms are combined in a (positive) weighted sum to give the full cost function
\begin{align}
    \label{eq:frenet_cost}
    \costFunctional{\trajC{\state}}{\trajC{\control}} &\defeq \costWeightProgress \costProgress{\trajC{\state}} + \nonumber \\ 
    \int_{0}^{\timeHorizon} \costWeightCentering & \costCentering{\trajC{\state}{\timeInstantRef}} + \costWeightBoundary \costBoundary{\trajC{\state}{\timeInstantRef}} + \nonumber \\ 
    \costWeightVelRef & \costVelRef{\trajC{\state}{\timeInstantRef}} + \costWeightVelNeg \costVelNeg{\trajC{\state}{\timeInstantRef}} \, \mathrm{d}\timeInstantRef.
\end{align}


\begin{remark}
    \label{rmk:discretized_raceline}
    When converting from Cartesian to Frenet coordinates, and at several points while evaluating~\eqref{eq:frenet_cost}, it is necessary to identify the desired velocity $\velLongRef$ or track boundary distances $\interval{\errLat}$ corresponding to the current progress $\progress$. 
    This is easily accomplished for individual trajectories, but necessitates additional consideration for intervals. 
    In principle, and with a closed-form expression for the observability information of Assumption~\ref{asm:track_geometry}, the inclusion function machinery of Section~\ref{subsec:reachability} is sufficient. 
    For computational speed, and to more closely model how Assumption~\ref{asm:track_geometry} may be satisfied in practice, we provide track geometry as a discretized array of points and approximate the conversion of an interval by considering trajectories corresponding to each of its four vertices in the position plane separately and collecting the results. 
    Since the reachable sets do not grow too large over the prediction horizon, if none of these trajectories result in collision, it is a strong indication that the safety check in~\eqref{eq:guaranteed_safe_samples} will be satisfied. 
    We empirically validate this approach in the sequel. 
\end{remark}

\subsection{Simulations}
\label{subsec:simulation}

With the dynamics described in Section~\ref{subsec:dynamics} as a case study, we performed a simulated, parameter sweep comparison between the reachability-aware Algorithm~\ref{alg:reachset_planner} and the baseline, trajectory only Algorithm~\ref{alg:sample_planner}.
Both controllers were tasked with driving three consecutive laps around the track shown in Figure \ref{fig:hardware_overlay} under artificially injected state disturbances.
While Algorithm \ref{alg:sample_planner} operates without knowledge of these disturbances, Algorithm \ref{alg:reachset_planner} is provided valid bounds on their maximum magnitude.
Neither method was given the specific disturbance realizations \emph{a priori}.

We conducted a parameter sweep comprising 30 initial states, 9 cost term weighting combinations, and 2 types of disturbance signals (540 total runs per controller). 
Trials were terminated early upon collision (``crash'') or if the vehicle failed to advance for 10 consecutive seconds (``stall'').

The reachability-based formulation demonstrated a marked improvement in safety.
Algorithm \ref{alg:sample_planner} crashed 141 times, compared to only 26 for Algorithm \ref{alg:reachset_planner}.
Notably, 25 of the 26 crashes under Algorithm \ref{alg:reachset_planner} occurred when the vehicle was initialized precariously close to the track boundary and never sampled a safe control, strongly suggesting that the safe control set was empty ($\safeControls{\stateInit} = \emptyset$).
Excluding these unrecoverable initializations, incorporating reachability reduced the crash frequency by over 99.1\%.
We suspect that the last remaining crash was caused by a failure to guarantee recursive feasibility, as described in Section~\ref{subsec:limitations}. 

\begin{figure}
    \centering
    \includegraphics[width=\linewidth]{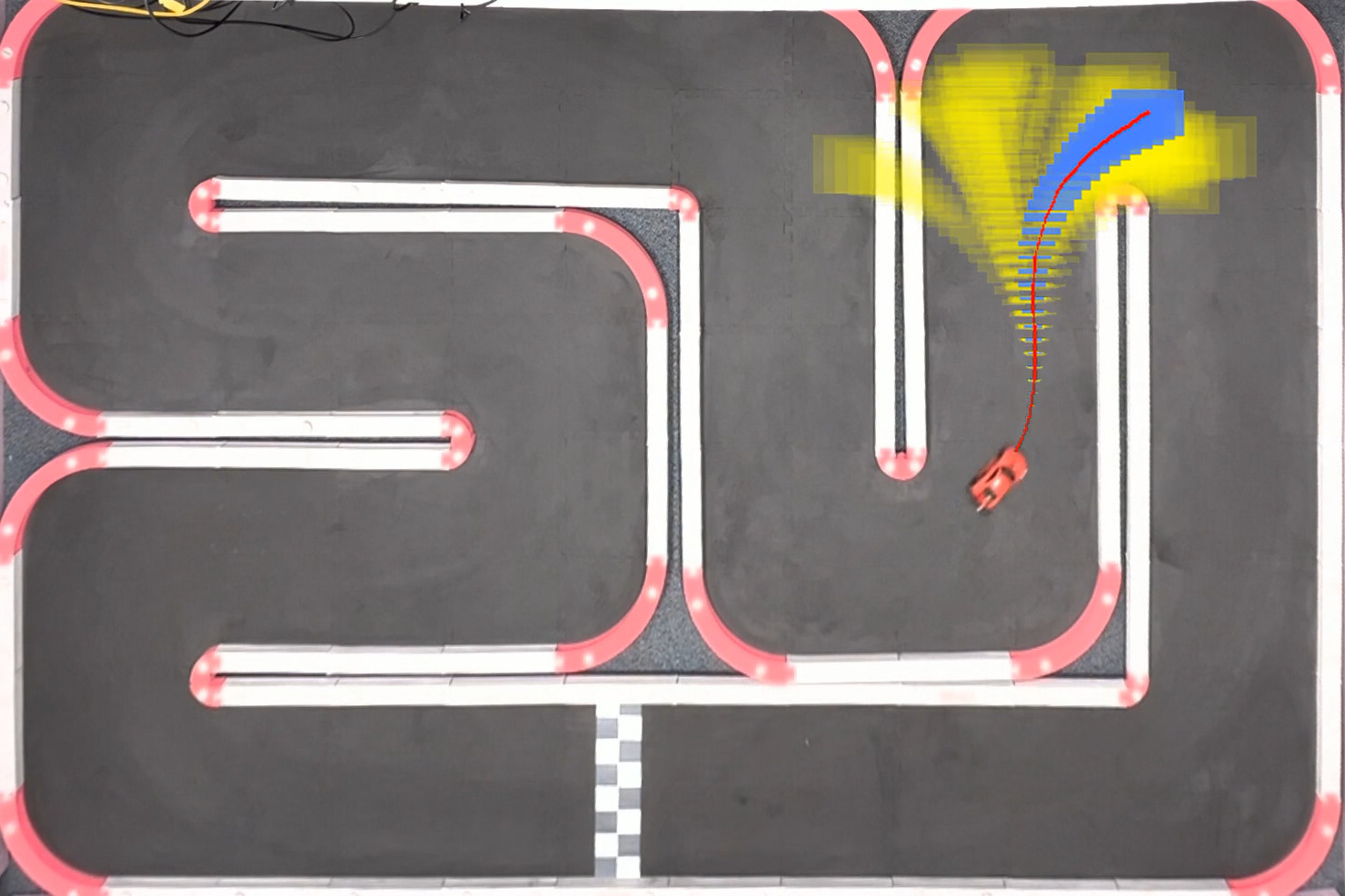}
    \caption{%
        Snapshot of Algorithm~\ref{alg:reachset_planner} being deployed on a \nicefrac{1}{28}th scale vehicle test platform.
        Computed over-approximations of the system's reachable sets are shown for a selection of the sampled controls, with the optimal control in blue and others in yellow. 
        The centerpoint trajectory of the planned reachable sets is shown in red. 
    }
    \label{fig:hardware_overlay}
\end{figure}

In addition to greatly reducing crash frequency, Algorithm~\ref{alg:reachset_planner} displays greater robustness to parameter variation than Algorithm~\ref{alg:sample_planner} (Table~\ref{tab:sweep-success-summary}).
The success rate is largely independent of disturbance realization and cost function weights, de-coupling safety guarantee from performance tuning. 
Indeed, by robustly accounting for disturbances, the safety considerations of Algorithm 2 had an emergent, positive impact on performance, resulting in universally faster lap times.


\begin{table*}
  \vspace{2mm}
  \centering
  \caption{%
      Outcome of the 1080-run sweep, broken down by sweep variable and reachable-set filter setting.
      Shaded rows are runs with the reachable-set filter enabled.
      Counts are runs with outcome \texttt{race\_finished}; lap times are means over exactly those runs, and ``---'' marks cells with no successful runs.
      Each ``Overall'' cell aggregates 540 runs, each disturbance-direction cell 270 runs, and each centering-weight and velocity-scale cell 180 runs.
  }
  \label{tab:sweep-success-summary}
  \begin{tabular}{lrrrrrrrrr}
    \toprule
    & & \multicolumn{2}{c}{Disturbance type}
      & \multicolumn{3}{c}{Centering weight $\costWeightCentering$ }
      & \multicolumn{3}{c}{Velocity scale} \\
    \cmidrule(lr){3-4} \cmidrule(lr){5-7} \cmidrule(lr){8-10}
    Reachable-set filter & Overall & Adversarial & Uniform
      & 0.1 & 0.5 & 1.0 & 1.00 & 1.25 & 1.50 \\
    \midrule
    \multicolumn{10}{l}{\textit{Successful runs}} \\
    \quad Disabled &  44 &  14 &  30 &   0 &   0 &  44 &  44 &   0 &   0 \\
    \rowcolor{gray!8}
    \quad Enabled  & 504 & 247 & 257 & 166 & 168 & 170 & 171 & 172 & 161 \\
    \addlinespace
    \multicolumn{10}{l}{\textit{Mean lap time (s)}} \\
    \quad Disabled & 5.97 & 6.02 & 5.94 & ---  & ---  & 5.97 & 5.97 & ---  & ---  \\
    \rowcolor{gray!8}
    \quad Enabled  & 5.52 & 5.57 & 5.48 & 5.73 & 5.49 & 5.36 & 5.63 & 5.42 & 5.52 \\
    \bottomrule
  \end{tabular}
\end{table*}

\subsection{Comparison}
\label{subsec:comparison}

We conducted a bi-directional, simulation-based comparison against the state-of-the-art reachability enabled planner~\cite{borquezDualGuardMPPISafe2025}. 
Our method achieves similar performance, maintains safety, and is scalable to larger system dimensions without any pre-computation. 

In~\cite[Section VI]{borquezDualGuardMPPISafe2025}, an HJ reachability-based method is used to navigate a racecar (modeled as a 3D Dubins vehicle) around a racetrack. 
We replicate this experiment in simulation using the same track geometry and system dynamics. 
Our algorithm achieves the same safety criteria (preventing all collisions) as the original work, and drives the vehicle at an average speed of 0.98 m/s. 
Compared to the baseline's reported 1.04 m/s, we incur a 6\% speed penalty. 
As discussed in~\cite{borquezDualGuardMPPISafe2025}, the rich information stored in the HJ value function is used to \emph{filter} every rollout, improving the sample efficiency of finding an optimal control and therefore performance overall. 
No such information is included in our interval reachable sets. 
However, computing this value function comes at a heavy cost as it must be evaluated and stored on a dense grid over the state space. 
On a system equipped with a Intel Xeon Gold 6230 CPU and NVIDIA Quadro RTX 8000 GPU, this took over 38 hours and the resulting file was 719.1 MB. 

Such overhead grows exponentially in system dimension, making it difficult to apply the state-of-the-art method to our 6D model from Section~\ref{subsec:dynamics}. 
To compute the value function for the system and environment pictured in Figure~\ref{fig:hardware_overlay}, it was necessary to reduce the granularity of the evaluation grid, and even then, the computation took over 115 hours on the same high-powered system, and the solution occupied 2.4 GB. 
Moreover, due to imprecision in interpolating between the stored (coarse) evaluations of the value function on our 6D vehicle, the HJ-based controller crashed before completing the same 3-lap trial developed in Secion~\ref{subsec:simulation}. 




\subsection{Experiment}
\label{subsec:real_world}

Finally, we validated Algorithm~\ref{alg:reachset_planner} on actual hardware. 
For these experiments, we used the palm-sized, remote control vehicle platform Buzzracer developed in prior work~\cite{zhangBuzzRacerPalmsizedAutonomous2024}. 
The controller was executed on a computer equipped with an Intel i9-10920X CPU and NVIDIA RTX3090 GPU, and inputs were broadcast wirelessly to the vehicle over a local area network (LAN). 
Algorithm~\ref{alg:reachset_planner} was configured to sample $\numControlSamples=1024$ control signals and predict $\predictionHorizon=30$ integration steps (of size $\controlDt=0.02$ seconds) for each sample every control loop. 

By leveraging heavy parallelization with JAX~\cite{jax2018github}, the control loop ran at nearly 60Hz. 
The vehicle completed over 35 consecutive laps on the track pictured in Figure~\ref{fig:hardware_overlay} without collision. 
These results validate both the disturbance structure chosen in~\eqref{eq:dyn_dynamic} and Algorithm~\ref{alg:reachset_planner} itself as applicable to the real world.

\section{Conclusion}
Inspired by prior work using reachability as a control filter for sampling-based MPC, we leverage fast, interval-based methods to realize full reachable sets for each sample online and guarantee safety over a receding horizon. 
With this additional information, the proposed controller was able to greatly reduce safety violations while meeting performance objectives, both in simulation and on hardware. 
Future work may attempt to further strengthen our safety guarantees by developing a backup policy to preserve recursive feasibility in the event that no certifiably safe control is sampled.
Alternatively, it could take advantage of the proposed online approach to address problems in dynamic environments or using adaptive control.



\section*{Acknowledgments}
The authors wish to thank Dr. Javier Borquez for generously providing clarifying comments on and detailed experimental data regarding~\cite{borquezDualGuardMPPISafe2025} that aided in development of the comparisons presented in Section~\ref{subsec:comparison}.



\printbibliography

\vfill

\end{document}

%% file: packages.tex
\usepackage{svg}
\usepackage{amsmath,amsfonts,amsthm}
\usepackage{mathrsfs}
\usepackage[noEnd=true]{algpseudocodex}
\usepackage{algorithm}

\usepackage{array}
\usepackage{textcomp}
\usepackage{stfloats}
\usepackage{url}
\usepackage{verbatim}
\usepackage{graphicx}
\usepackage{subcaption}
\usepackage[dvipsnames,table]{xcolor}
\definecolor{forestgreen}{RGB}{34,139,34}
\usepackage{nicefrac}
\usepackage{booktabs}
\usepackage{hyperref}

\usepackage[backend=biber, style=ieee]{biblatex}
\AtEveryBibitem{%
  \clearfield{url}%
  \clearfield{doi}%
  \clearfield{isbn}%
  \clearfield{issn}%
  \clearfield{note}%
  \clearfield{urlyear}%
  \clearfield{eprinttype}%
  \clearfield{eprint}%
  \clearlist{language}%
}
\DeclareCaseLangs{}

\theoremstyle{definition}
\newtheorem{theorem}{Theorem}
\newtheorem{proposition}{Proposition}

\newtheorem{problem}{Problem}
\newtheorem{definition}{Definition}
\newtheorem{assumption}{Assumption}
\newtheorem{remark}{Remark}

%% file: macros.tex
\makeatletter
\let\LB@pending\relax
\newcommand{\LBname}[1]{%
  \ifx\LB@pending\relax
    #1%
  \else
    \let\LB@apply\LB@pending
    \let\LB@pending\relax      
    \LB@apply{#1}%
  \fi}

\def\LBend{\LB@stop}\def\LB@stop{}
\def\LB@clear{\let\LB@pending\relax}
\newif\ifLB@found
\def\LB@ifalph#1#2#3{
  \LB@foundfalse
  \ifx#1\mathcal \LB@foundtrue\fi \ifx#1\mathbf \LB@foundtrue\fi
  \ifx#1\mathrm  \LB@foundtrue\fi \ifx#1\mathbb \LB@foundtrue\fi
  \ifx#1\mathsf  \LB@foundtrue\fi \ifx#1\mathit \LB@foundtrue\fi
  \ifx#1\mathtt  \LB@foundtrue\fi
  \ifLB@found #2\else #3\fi}
\def\LB@scan#1#2#3{
  \ifx#3\LBend
    \def\LB@next{#1{#2}}%
  \else
    \ifcat\noexpand#3a%
      \def\LB@next{\LB@scan{#1}{#2#3}}%
    \else
      \def\LB@tn{#2}%
      \ifx\LB@tn\@empty
        \def\LB@next{\LB@head{#1}#3}%
      \else
        \def\LB@next{#1{#2}\LB@clear #3}%
      \fi
    \fi
  \fi
  \LB@next}
\def\LB@head#1#2{%
  \ifcat\noexpand#2\relax
    \LB@ifalph{#2}%
      {\def\LB@hnext{\LB@alphmaybe{#1}{#2}}}
      {\def\LB@hnext{#2}}
  \else                                      
    \def\LB@hnext{#2}%
  \fi
  \LB@hnext}
\def\LB@alphmaybe#1#2{\@ifnextchar\bgroup{\LB@alphg{#1}{#2}}{#1{#2}\LB@clear}}
\def\LB@alphg#1#2#3{#1{#2{#3}}\LB@clear}

\let\LB@affix\@empty

\newcommand{\LBnameaffix}[2]{
  \ifx\LB@pending\relax
    #1#2
  \else
    \let\LB@apply\LB@pending \let\LB@pending\relax
    \def\LB@affix{#2}%
    \LB@apply{#1}%
    \let\LB@affix\@empty
  \fi}

\newcommand{\LBunderlinehead}[1]{\underline{#1}\LB@affix}
\newcommand{\LBoverlinehead}[1]{\overline{#1}\LB@affix}
\newcommand{\LBintervalhead}[1]{\left[\underline{#1}\LB@affix, \overline{#1}\LB@affix\right]}

\newcommand{\LB@compose}[1]{%
  \ifx\LB@pending\relax
    \let\LB@pending#1%
  \else
    \let\LB@oldpending\LB@pending
    \def\LB@pending##1{\LB@oldpending{#1{##1}}}
  \fi}
\newcommand{\LB@runscan}[1]{\LB@scan{\LB@pending}{}#1\LBend}

\newcommand{\LBdecoratehead}[2]{\begingroup \LB@compose{#1}\LB@runscan{#2}\endgroup}

\newcommand{\NewLetterWrapCommand}[2]{%
  \protected\def#1##1{\LBdecoratehead{#2}{##1}}}

\newcommand{\LBsubscripthead}[2]{
  \begingroup
    \def\LB@this##1{\mathnormal{##1}_{#1}}%
    \LB@compose{\LB@this}\LB@runscan{#2}%
  \endgroup}

\newcommand{\LBsuperscripthead}[2]{
  \begingroup
    \def\LB@this##1{{##1}^{#1}}%
    \LB@compose{\LB@this}\LB@runscan{#2}%
  \endgroup}
\makeatother

\NewLetterWrapCommand{\lb}{\LBunderlinehead}
\NewLetterWrapCommand{\ub}{\LBoverlinehead}
\NewDocumentCommand{\interval}{mg}{%
  \IfNoValueTF{#2}{\LBdecoratehead{\LBintervalhead}{#1}}{\left[#1, #2\right]}}

\newcommand{\componentIndex}{k}
\NewDocumentCommand{\component}{mg}{%
  \IfNoValueTF{#2}{\LBsubscripthead{\componentIndex}{#1}}{\LBsubscripthead{#2}{#1}}}

\newcommand{\sampleIndex}{i}
\newcommand{\sampleIndexReference}{i^{*}}
\NewDocumentCommand{\sample}{mg}{%
  \IfNoValueTF{#2}{\LBsuperscripthead{(\sampleIndex)}{#1}}{\LBsuperscripthead{\left(#2\right)}{#1}}}

\newcommand{\timeInstant}{t}
\newcommand{\timeInstantRef}{\tau}
\NewDocumentCommand{\trajC}{mg}{\IfNoValueTF{#2}{#1(\cdot)}{#1\left(#2\right)}}
\NewDocumentCommand{\trajD}{mg}{\IfNoValueTF{#2}{#1[\cdot]}{#1\left[#2\right]}}

\newcommand{\R}{\mathbb{R}}
\newcommand{\IR}{\mathbb{IR}}
\newcommand{\functionSet}[2]{L\left({#1}; {#2}\right)}

\newcommand{\defeq}{:=}
\newcommand{\boundary}{\partial}

\DeclareMathOperator{\clip}{clip}
\DeclareMathOperator{\dist}{dist}

\NewDocumentCommand{\dyn}{ggg}{\IfNoValueTF{#1}{\LBname{f}}{\LBname{f}\left(#1, #2, #3\right)}}

\newcommand{\timeHorizon}{T}
\newcommand{\traj}[4]{\LBname{\phi}\left(#1; #2; #3, #4\right)}
\newcommand{\embTraj}[3]{\interval{\LBname{\Phi}\left(#1; #2; #3\right)}}

\NewDocumentCommand{\safetyFn}{g}{\IfNoValueTF{#1}{\LBname{h}}{\LBname{h}\left(#1\right)}}
\newcommand{\safeSymbol}{\mathcal{S}}
\newcommand{\safeStates}{{\safeSymbol_{\state}}}
\newcommand{\safeControls}[1]{{\safeSymbol_{\control}^{\timeHorizon}}\left(#1\right)}
\newcommand{\maybeSafeSamples}[1]{\hat{\safeSymbol}_{\mathcal{I}}^{\timeHorizon}\left(#1\right)}
\newcommand{\safeSamples}[1]{\safeSymbol_{\mathcal{I}}^{\timeHorizon}\left(#1\right)}

\newcommand{\numControlSamples}{N}
\newcommand{\predictionHorizon}{M}

\newcommand{\state}{\LBname{\mathbf{x}}}
\newcommand{\stateDim}{{d_{\state}}}
\newcommand{\stateInit}{\state_{\mathrm{init}}}

\newcommand{\control}{\LBname{\mathbf{u}}}

\newcommand{\controlRef}{\control^{*}}
\newcommand{\controlDim}{{d_{\control}}}
\newcommand{\controlDt}{\Delta_{\control} \timeInstant}
\newcommand{\controlDiff}{\Delta \control}

\newcommand{\err}{\LBname{\mathbf{w}}}
\newcommand{\errDim}{{d_{\err}}}

\newcommand{\gravity}{g}

\newcommand{\carLenFront}{\ell_{\mathrm{f}}}
\newcommand{\carLenRear}{\ell_{\mathrm{r}}}
\newcommand{\carSlipFront}{F_{\mathrm{f}}}
\newcommand{\carSlipRear}{F_{\mathrm{r}}}
\newcommand{\carMass}{m}
\newcommand{\carRotationalInertia}{I_z}

\newcommand{\motorModel}[1]{m_{\mathrm{acc}}\left(#1\right)}
\newcommand{\motorAccWeight}{C_a}
\newcommand{\motorVelEfficiency}{C_v}
\newcommand{\motorEquilibrium}{C_e}

\newcommand{\tireSlipCurve}[1]{t\left(#1\right)}
\newcommand{\tireSlipShape}{C}
\newcommand{\tireSlipPeak}{D}
\newcommand{\tireSlipStiffness}{B}

\newcommand{\posX}{x}
\newcommand{\posY}{y}
\newcommand{\heading}{\phi}
\newcommand{\velLong}{v_{\mathrm{long}}}
\newcommand{\velLongRef}{v_{\mathrm{long}}^{*}}
\newcommand{\velLat}{v_{\mathrm{lat}}}
\newcommand{\velHeading}{\omega}

\newcommand{\progress}{s}
\newcommand{\progressRef}{\progress^{*}}
\newcommand{\errLat}{\LBnameaffix{e}{_{\mathrm{lat}}}}
\newcommand{\errHeading}{\LBnameaffix{e}{_{\mathrm{\phi}}}}
\newcommand{\errVelHeading}{\LBnameaffix{e}{_{\mathrm{\omega}}}}

\newcommand{\inputThrottle}{u_{\mathrm{acc}}}
\newcommand{\inputSteer}{u_{\mathrm{steer}}}

\newcommand{\distVelLat}{\LBnameaffix{w}{_{\velLat}}}
\newcommand{\distHeading}{\LBnameaffix{w}{_{\heading}}}

\NewDocumentCommand{\costFunctional}{gg}{\IfNoValueTF{#1}{\LBname{J}}{\LBname{J}\left(#1, #2\right)}}
\NewDocumentCommand{\costFunctionalInt}{gg}{\IfNoValueTF{#1}{\LBname{J_{\mathrm{int}}}}{\LBname{J_{\mathrm{int}}}\left(#1, #2\right)}}
\newcommand{\costCentering}[1]{\costFunctional_{\mathrm{center}} \left( #1 \right)}
\newcommand{\costBoundary}[1]{\costFunctional_{\mathrm{bnd}} \left( #1 \right)}
\newcommand{\costVelRef}[1]{\costFunctional_{\mathrm{velR}} \left( #1 \right)}
\newcommand{\costVelNeg}[1]{\costFunctional_{\mathrm{velN}} \left( #1 \right)}
\newcommand{\costProgress}[1]{\costFunctional_{\mathrm{prog}} \left( #1 \right)}

\newcommand{\costWeightCentering}{\gamma_{\mathrm{center}}}
\newcommand{\costWeightBoundary}{\gamma_{\mathrm{bnd}}}
\newcommand{\costWeightVelRef}{\gamma_{\mathrm{velR}}}
\newcommand{\costWeightVelNeg}{\gamma_{\mathrm{velN}}}
\newcommand{\costWeightProgress}{\gamma_{\mathrm{prog}}}

\DeclareMathOperator*{\argmin}{\arg\,\min}

\newcommand{\flatten}[2]{\component{\LBname{\mathcal{B}}\left(#2\right)}{#1}}